\documentclass[reqno]{amsart}

\usepackage{graphicx}
\usepackage{multirow}
\usepackage{hyperref}
\hypersetup{colorlinks=true, linkcolor=blue, urlcolor=blue, linktocpage=true}

\theoremstyle{plain}
\newtheorem{theorem}{Theorem}[section]

\theoremstyle{remark}
\newtheorem{remark}{Remark}[section]

\numberwithin{equation}{section}
\numberwithin{figure}{section}
\numberwithin{table}{section}

\title[On the approximation by single hidden layer networks]{On the approximation by single hidden layer feedforward neural networks with fixed weights}

\author{Namig J. Guliyev}
\address{Institute of Mathematics and Mechanics, Azerbaijan National Academy of Sciences, 9 B.~Vahabzadeh str., AZ1141, Baku, Azerbaijan.}
\email{njguliyev@gmail.com}

\author{Vugar E. Ismailov}
\address{Institute of Mathematics and Mechanics, Azerbaijan National Academy of Sciences, 9 B.~Vahabzadeh str., AZ1141, Baku, Azerbaijan.}
\email{vugaris@mail.ru}



\begin{document}
\maketitle
\begin{abstract}
Feedforward neural networks have wide applicability in various disciplines of science due to their universal approximation property. Some authors have shown that single hidden layer feedforward neural networks (SLFNs) with fixed weights still possess the universal approximation property provided that approximated functions are univariate. But this phenomenon does not lay any restrictions on the number of neurons in the hidden layer. The more this number, the more the probability of the considered network to give precise results. In this note, we constructively prove that SLFNs with the fixed weight $1$ and two neurons in the hidden layer can approximate any continuous function on a compact subset of the real line. The applicability of this result is demonstrated in various numerical examples. Finally, we show that SLFNs with fixed weights cannot approximate all continuous multivariate functions.
\end{abstract}

\tableofcontents

\section{Introduction} \label{sec:introduction}

Approximation capabilities of single hidden layer feedforward neural networks (SLFNs) have been investigated in many works over the past 30 years. Typical results show that SLFNs possess the universal approximation property; that is, they can approximate any continuous function on a compact set with arbitrary precision.

An SLFN with $r$ units in the hidden layer and input $\mathbf{x} = (x_{1}, \ldots, x_{d})$ evaluates a function of the form
\begin{equation} \label{eq:intro}
  \sum_{i=1}^r c_i \sigma(\mathbf{w}^i \cdot \mathbf{x} - \theta_i),
\end{equation}
where the weights $\mathbf{w}^i$ are vectors in $\mathbb{R}^{d}$, the thresholds $\theta_i$ and the coefficients $c_i$ are real numbers, and the activation function $\sigma$ is a univariate function. Properties of this neural network model have been studied quite well. By choosing various activation functions, many authors proved that SLFNs with the chosen activation function possess the universal approximation property (see, e.g., \cite{CC93, CL92, CS13, C90, C89, F89, GW88, H91, MM92}).
That is, for any compact set $Q \subset \mathbb{R}^{d}$, the class of functions~(\ref{eq:intro}) is dense in $C(Q)$, the space of continuous functions on $Q$.
The most general and complete result of this type was obtained by Leshno,
Lin, Pinkus and Schocken \cite{LLPS93}. They proved that a continuous activation function $\sigma$ has the universal approximation property (or density property) if and only if it is not a polynomial. This result has shown the power of SLFNs within all possible choices of the activation function $\sigma$, provided that $\sigma$ is continuous. For a detailed review of these and many other results, see \cite{P99}.

In many applications, it is convenient to take the activation function $\sigma$ as a \emph{sigmoidal function} which is defined as
\begin{equation*}
  \lim_{t \to -\infty} \sigma(t) = 0 \quad \text{ and } \quad \lim_{t \to +\infty} \sigma(t) = 1.
\end{equation*}
The literature on neural networks abounds with the use of such functions and their superpositions (see, e.g., \cite{CX10, CL92, CS13, C89, F89, GW88, HH04, IKM17, I92, K92, MM92}). The possibility of approximating a continuous function on a compact subset of the real line or $d$-dimensional space by SLFNs with a sigmoidal activation function has been well studied in a number of papers.

In recent years, the theory of neural networks has been developed further in this direction. For example, from the point of view of practical applications, neural networks with a restricted set of weights have gained a special interest (see, e.g., \cite{D02, I12, I15, IS17, JYJ10, LFN04}). It was proved that SLFNs with some restricted set of weights still possess the universal approximation property.
For example, Stinchcombe and White \cite{SW90} showed that SLFNs with a polygonal, polynomial spline or analytic activation function and a bounded set of weights have the universal approximation property. Ito \cite{I92} investigated this property of networks using monotone sigmoidal functions (tending to $0$ at minus infinity and $1$ at infinity), with only weights located on the unit sphere. In \cite{I12, I15, IS17}, one of the coauthors considered SLFNs with weights varying on a restricted set of directions and gave several necessary and sufficient conditions for good approximation by such networks. For a set $W$ of weights consisting of two directions, he showed that there is a geometrically explicit solution to the problem. Hahm and Hong \cite{HH04} went further in this direction, and showed that SLFNs with fixed weights can approximate arbitrarily well any univariate function. Since fixed weights reduce the computational expense and training time, this result is of particular interest. In a mathematical formulation, the result reads as follows.

\begin{theorem}[Hahm and Hong \cite{HH04}] \label{thm:HH}
Assume $f$ is a continuous function on a finite segment $[a, b]$ of $\mathbb{R}$. Assume $\sigma$ is a bounded measurable sigmoidal function on $\mathbb{R}$. Then for any sufficiently small $\varepsilon > 0$ there exist constants $c_i$, $\theta_i \in \mathbb{R}$ and positive integers $K$ and $n$ such that
\begin{equation*}
  \left\vert f(x) - \sum_{i=1}^{n} c_i \sigma(Kx - \theta_i) \right\vert < \varepsilon
\end{equation*}
for all $x \in [a, b]$.
\end{theorem}

Note that in this theorem both $K$ and $n$ depend on $\varepsilon$. The smaller the $\varepsilon$, the more neurons in the hidden layer one should take to approximate with the required precision. This phenomenon is pointed out as necessary in many papers. For various activation functions $\sigma$, there are plenty of practical examples, diagrams, tables, etc. in the literature, showing how the number of neurons increases as the error of approximation gets smaller.

It is well known that one of the challenges of neural networks is the process of deciding optimal number of hidden neurons. The other challenge is understanding how to reduce the computational expense and training time. As usual, networks with fixed weights best fit this purpose. In this respect, Cao and Xie \cite{CX10} strengthened the above result by specifying the number of hidden neurons to realize approximation to any continuous function. By implementing modulus of continuity, they established upper bound estimations for the approximation error. It was shown in \cite{CX10} that for the class of Lipschitz functions $\operatorname{Lip}_M(\alpha)$ with a Lipschitz constant $M$ and degree $\alpha$, the approximation bound is $M(1+\left\Vert \sigma \right\Vert )(b-a)n^{-\alpha }$, where $\left\Vert \sigma \right\Vert$ is the $\sup$ of $\sigma(x)$ on $[a, b]$.

Approximation capabilities of SLFNs with a fixed weight were also analyzed in Lin, Guo, Cao and Xu \cite{LGCX13}. Taking the activation function $\sigma$ as a continuous, even and $2 \pi$-periodic function, the authors of \cite{LGCX13} showed that neural networks of the form
\begin{equation} \label{eq:intro2}
  \sum_{i=1}^r c_i \sigma(x - x_i)
\end{equation}
can approximate any continuous function on $[-\pi, \pi]$ with an arbitrary precision $\varepsilon$. Note that all the weights are fixed equal to $1$, and consequently do not depend on $\varepsilon$. To prove this, they first gave an integral representation for trigonometric polynomials, and constructed explicitly a network formed as (\ref{eq:intro2}) that approximates this integral representation. Finally, the obtained result for trigonometric polynomials was used to prove a Jackson-type upper bound for the approximation error.

In this paper, we construct a special sigmoidal activation function which meets both the above mentioned challenges in the univariate setting. In mathematical terminology, we construct a sigmoidal function $\sigma$ for which $K$ and $n$ in the above theorem do not depend on the error $\varepsilon$. Moreover, we can take $K=1$ and $n=2$. That is, only parameters $c_i$ and $\theta_i$ depend on $\varepsilon$. Can we find these numbers? For a large class of functions $f$, especially for analytic functions, our answer to this question is positive. We give an algorithm and a computer program for computing these numbers in practice. Our results are illustrated by several examples. Finally, we show that SLFNs with fixed weights are not capable of approximating all multivariate functions with arbitrary precision.

\section{Construction of a sigmoidal function} \label{sec:construction}

In this section, we construct algorithmically a sigmoidal function $\sigma$
which we use in our main result in the following section. Besides sigmoidality, we take care about smoothness and monotonicity of our $\sigma$ in the weak sense. Here by ``weak monotonicity" we understand behavior of a function whose difference in absolute value from a monotonic function is a sufficiently small number. In this regard, we say that a real function $f$
defined on a set $X\subseteq $ $\mathbb{R}$ is called \emph{$\lambda$-increasing} (respectively, \emph{$\lambda$-decreasing}) if there exists an increasing (respectively, decreasing) function $u \colon X \to \mathbb{R}$ such that $\left\vert f(x)-u(x)\right\vert
\leq \lambda$ for all $x \in X$. Obviously, $0$-monotonicity coincides with the usual concept of monotonicity, and a $\lambda _{1}$-increasing function is $\lambda _{2}$-increasing if $\lambda _{1}\leq \lambda _{2}$.

To start with the construction of $\sigma$, assume that we are given a closed interval $[a, b]$ and a sufficiently small real number $\lambda$. We construct $\sigma$ algorithmically, based on two numbers, namely $\lambda$ and $d := b - a$. The following steps describe the algorithm.

1. Introduce the function
\begin{equation*}
  h(x) := 1 - \frac{\min\{1/2, \lambda\}}{1 + \log(x - d + 1)}.
\end{equation*}
Note that this function is strictly increasing on the real line and satisfies the following properties:
\begin{enumerate}
  \item $0 < h(x) < 1$ for all $x \in [d, +\infty)$;
  \item $1 - h(d) \le \lambda$;
  \item $h(x) \to 1$, as $x \to +\infty$.
\end{enumerate}

We want to construct $\sigma$ satisfying the inequalities
\begin{equation} \label{eq:h_sigma_1}
  h(x) < \sigma(x) < 1
\end{equation}
for $x \in [d, +\infty)$. Then our $\sigma$ will tend to $1$
as $x$ tends to $+\infty $ and obey the inequality
\begin{equation*}
  |\sigma(x) - h(x)| \le \lambda,
\end{equation*}
i.e., it will be a $\lambda$-increasing function.

2. Before proceeding to the construction of $\sigma$, we need to enumerate the monic polynomials with rational coefficients. Let $q_n$ be the Calkin--Wilf sequence (see~\cite{CW00}). Then we can enumerate all the rational numbers by setting
$$r_0 := 0, \quad r_{2n} := q_n, \quad r_{2n-1} := -q_n, \ n = 1, 2, \dots.$$
Note that each monic polynomial with rational coefficients can uniquely be written as $r_{k_0} + r_{k_1} x + \ldots + r_{k_{l-1}} x^{l-1} + x^l$, and each positive rational number determines a unique finite continued fraction
$$
  [m_0; m_1, \ldots, m_l] := m_0 + \dfrac1{m_1 + \dfrac1{m_2 + \dfrac1{\ddots + \dfrac1{m_l}}}}
$$
with $m_0 \ge 0$, $m_1, \ldots, m_{l-1} \ge 1$ and $m_l \ge 2$. We now construct a bijection between the set of all monic polynomials with rational coefficients and the set of all positive rational numbers as follows. To the only zeroth-degree monic polynomial 1 we associate the rational number 1, to each first-degree monic polynomial of the form $r_{k_0} + x$ we associate the rational number $k_0 + 2$, to each second-degree monic polynomial of the form $r_{k_0} + r_{k_1} x + x^2$ we associate the rational number $[k_0; k_1 + 2] = k_0 + 1 / (k_1 + 2)$, and to each monic polynomial
$$
  r_{k_0} + r_{k_1} x + \ldots + r_{k_{l-2}} x^{l-2} + r_{k_{l-1}} x^{l-1} + x^l
$$
of degree $l \ge 3$ we associate the rational number $[k_0; k_1 + 1, \ldots, k_{l-2} + 1, k_{l-1} + 2]$. In other words, we define $u_1(x) := 1$,
$$
  u_n(x) := r_{q_n-2} + x
$$
if $q_n \in \mathbb{Z}$,
$$
  u_n(x) := r_{m_0} + r_{m_1-2} x + x^2
$$
if $q_n = [m_0; m_1]$, and
$$
  u_n(x) := r_{m_0} + r_{m_1-1} x + \ldots + r_{m_{l-2}-1} x^{l-2} + r_{m_{l-1}-2} x^{l-1} + x^l
$$
if $q_n = [m_0; m_1, \ldots, m_{l-2}, m_{l-1}]$ with $l \ge 3$. For example, the first few elements of this sequence are
$$
  1, \quad x^2, \quad x, \quad x^2 - x, \quad x^2 - 1, \quad x^3, \quad x - 1, \quad x^2 + x, \quad \ldots.
$$

3. We start with constructing $\sigma$ on the intervals $[(2n-1)d, 2nd]$, $n = 1, 2, \ldots$. For each monic polynomial $u_n(x) = \alpha_0 + \alpha_1 x + \ldots + \alpha_{l-1} x^{l-1} + x^l$, set
\begin{equation*}
  B_1 := \alpha_0 + \frac{\alpha_1-|\alpha_1|}{2} + \ldots + \frac{\alpha_{l-1} - |\alpha_{l-1}|}{2}
\end{equation*}
and
\begin{equation*}
  B_2 := \alpha_0 + \frac{\alpha_1+|\alpha_1|}{2} + \ldots + \frac{\alpha_{l-1} + |\alpha_{l-1}|}{2} + 1.
\end{equation*}
Note that the numbers $B_1$ and $B_2$ depend on $n$. To avoid complication of symbols, we do not indicate this in the notation.

Introduce the sequence
\begin{equation*}
  M_n := h((2n+1)d), \qquad n = 1, 2, \ldots.
\end{equation*}
Clearly, this sequence is strictly increasing and converges to $1$. 

Now we define $\sigma$ as the function
\begin{equation} \label{eq:sigma_again}
\sigma(x) := a_n + b_n u_n \left( \frac{x}{d} - 2n + 1 \right), \quad x \in [(2n-1)d, 2nd],
\end{equation}
where
\begin{equation} \label{eq:a_1, b_1}
  a_1 := \frac{1}{2}, \qquad b_1 := \frac{h(3d)}{2},
\end{equation}
and
\begin{equation} \label{eq:a_n, b_n}
  a_n := \frac{(1 + 2M_n) B_2 - (2 + M_n) B_1}{3(B_2 - B_1)}, \qquad b_n := \frac{1 - M_n}{3(B_2 - B_1)}, \qquad n = 2, 3, \ldots.
\end{equation}

It is not difficult to notice that for $n>2$ the numbers $a_n$, $b_n$ are the coefficients of the linear function $y = a_n + b_n x$ mapping the closed interval $[B_1, B_2]$ onto the closed interval $[(1+2M_n)/3, (2+M_n)/3]$. Besides, for $n=1$, i.e. on the interval $[d,2d]$,
\begin{equation*}
  \sigma(x) = \frac{1+M_1}{2}.
\end{equation*}
Therefore, we obtain that
\begin{equation} \label{eq:h_M_sigma_1}
  h(x) < M_n < \frac{1+2M_n}{3} \le \sigma(x) \le \frac{2+M_n}{3} < 1,
\end{equation}
for all $x \in [(2n-1)d, 2nd]$, $n = 1$, $2$, $\ldots$.

4. In this step, we construct $\sigma$ on the intervals $[2nd, (2n+1)d]$, $n = 1, 2, \ldots$. For this purpose we use the \emph{smooth transition function}
\begin{equation*}
  \beta_{a,b}(x) := \frac{\widehat{\beta}(b-x)}{\widehat{\beta}(b-x) + \widehat{\beta}(x-a)},
\end{equation*}
where
\begin{equation*}
  \widehat{\beta}(x) := \begin{cases} e^{-1/x}, & x > 0, \\ 0, & x \le 0. \end{cases}
\end{equation*}
Obviously, $\beta_{a,b}(x) = 1$ for $x \le a$, $\beta_{a,b}(x) = 0$ for $x \ge b$, and $0 < \beta_{a,b}(x) < 1$ for $a < x < b$. 

Set
\begin{equation*}
  K_n := \frac{\sigma(2nd) + \sigma((2n+1)d)}{2}, \qquad n = 1, 2, \ldots.
\end{equation*}
Note that the numbers $\sigma(2nd)$ and $\sigma((2n+1)d)$ have already been defined in the previous step. Since both the numbers $\sigma(2nd)$ and $\sigma((2n+1)d)$ lie in the interval $(M_n, 1)$, it follows that $K_n \in (M_n, 1)$.

First we extend $\sigma$ smoothly to the interval $[2nd, 2nd + d/2]$. Take $\varepsilon := (1 - M_n)/6$ and choose $\delta \le d/2$ such that
\begin{equation} \label{eq:epsilon}
  \left| a_n + b_n u_n \left( \frac{x}{d} - 2n + 1 \right) - \left( a_n + b_n u_n(1) \right) \right| \le \varepsilon, \quad x \in [2nd, 2nd + \delta].
\end{equation}
One can choose this $\delta$ as
\begin{equation*}
  \delta := \min\left\{ \frac{\varepsilon d}{b_n C}, \frac{d}{2} \right\},
\end{equation*}
where $C > 0$ is a number satisfying $|u'_n(x)| \le C$ for $x \in
(1, 1.5)$. For example, for $n=1$, $\delta$ can be chosen as $d/2$. Now define $\sigma$ on the first half of the interval $[2nd, (2n+1)d]$ as the function
\begin{equation} \label{eq:sigma_left}
\begin{split}
  \sigma(x) & := K_n - \beta_{2nd, 2nd + \delta}(x) \\
  & \times \left(K_n - a_n - b_n u_n \left( \frac{x}{d} - 2n + 1 \right)\right), \quad x \in \left[ 2nd, 2nd + \frac{d}{2} \right].
\end{split}
\end{equation}

Let us prove that $\sigma(x)$ satisfies the condition~(\ref{eq:h_sigma_1}). Indeed, if $2nd + \delta \le x \le 2nd + d/2$, then there is nothing to prove, since $\sigma(x) = K_n \in (M_n, 1)$. If $2nd \le x < 2nd + \delta$, then $0 < \beta_{2nd, 2nd+\delta}(x) \le 1$ and hence from~(\ref{eq:sigma_left}) it follows that for each $x \in [2nd, 2nd + \delta)$, $\sigma(x)$ is between the numbers $K_n$ and $A_n(x) := a_n + b_n u_n \left( \frac{x}{d} - 2n + 1 \right)$. On the other hand, from~(\ref{eq:epsilon}) we obtain that
\begin{equation*}
  a_n + b_n u_n(1) - \varepsilon \le A_n(x) \le a_n + b_n u_n(1) + \varepsilon,
\end{equation*}
which together with~(\ref{eq:sigma_again}) and~(\ref{eq:h_M_sigma_1}) yields $A_n(x) \in \left[ \frac{1+2M_n}{3} - \varepsilon, \frac{2+M_n}{3} + \varepsilon \right]$ for $x \in [2nd, 2nd + \delta)$. Since $\varepsilon = (1 - M_n)/6$, the inclusion $A_n(x) \in (M_n, 1)$ is valid. Now since both $K_n$ and $A_n(x)$ belong to $(M_n, 1)$, we finally conclude that
\begin{equation*}
  h(x) < M_n < \sigma(x) < 1, \quad \text{for } x \in \left[ 2nd, 2nd + \frac{d}{2} \right].
\end{equation*}

We define $\sigma$ on the second half of the interval in a similar way:
\begin{equation*}
\begin{split}
  \sigma(x) & := K_n - (1 - \beta_{(2n+1)d - \overline{\delta}, (2n+1)d}(x)) \\
  & \times \left(K_n - a_{n+1} - b_{n+1} u_{n+1} \left( \frac{x}{d} - 2n - 1 \right)\right), \quad x \in \left[ 2nd + \frac{d}{2}, (2n+1)d \right],
\end{split}
\end{equation*}
where
\begin{equation*}
  \overline{\delta} := \min\left\{ \frac{\overline{\varepsilon}d}{b_{n+1} \overline{C}}, \frac{d}{2} \right\}, \qquad \overline{\varepsilon} := \frac{1 - M_{n+1}}{6}, \qquad \overline{C} \ge \sup_{[-0.5, 0]} |u'_{n+1}(x)|.
\end{equation*}
One can easily verify, as above, that the constructed $\sigma(x)$ satisfies the condition~(\ref{eq:h_sigma_1}) on $[2nd + d/2, 2nd + d]$ and
\begin{equation*}
  \sigma \left( 2nd + \frac{d}{2} \right) = K_n, \qquad \sigma^{(i)} \left( 2nd + \frac{d}{2} \right) = 0, \quad i = 1, 2, \ldots.
\end{equation*}

Steps 3 and 4 construct $\sigma$ on the interval $[d, +\infty)$.

5. On the remaining interval $(-\infty, d)$, we define $\sigma$ as
\begin{equation*}
  \sigma(x) := \left( 1 - \widehat{\beta}(d-x) \right) \frac{1 + M_1}{2}, \quad x \in (-\infty, d).
\end{equation*}
It is not difficult to verify that $\sigma$ is a strictly increasing, smooth function on $(-\infty, d)$. Note also that $\sigma(x) \to \sigma(d) = (1 + M_1) / 2$, as $x$ tends to $d$ from the left and $\sigma^{(i)}(d) = 0$ for $i = 1$, $2$, $\ldots$. This final step completes the construction of $\sigma$ on the whole real line.

\section{Practical computation and properties of the constructed sigmoidal function} \label{sec:computation}

It should be noted that the above algorithm allows one to compute the constructed $\sigma$ at any point of the real axis instantly. The code of this algorithm is available at \url{http://sites.google.com/site/njguliyev/papers/monic-sigmoidal}. As a practical example, we give here the graph of $\sigma$ (see Figure~\ref{fig:sigma}) and a numerical table (see Table~\ref{tbl:sigma}) containing several computed values of this function on the interval $[0, 20]$. Figure~\ref{fig:sigma100} shows how the graph of $\lambda$-increasing function $\sigma$ changes on the interval $[0,100]$ as the parameter $\lambda$ decreases.

The above $\sigma$ obeys the following properties:
\begin{enumerate}
  \item $\sigma$ is sigmoidal;
  \item $\sigma \in C^{\infty}(\mathbb{R})$;
  \item $\sigma$ is strictly increasing on $(-\infty, d)$ and $\lambda$-strictly increasing on $[d, +\infty)$;
  \item $\sigma$ is easily computable in practice.
\end{enumerate}

All these properties are easily seen from the above exposition. But the essential property of our sigmoidal function is its ability to approximate an arbitrary continuous function using only a fixed number of translations and scalings of $\sigma$. More precisely, only two translations and scalings are sufficient. We formulate this important property as a theorem in the next section.

\begin{figure}
  \includegraphics[width=1.0\textwidth]{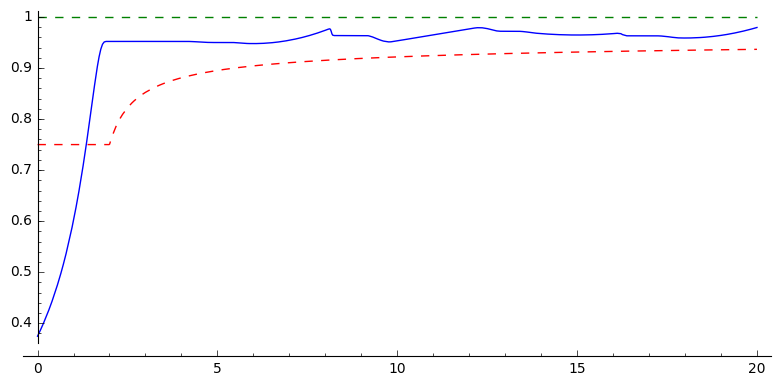}
  \caption{The graph of $\sigma$ on $[0, 20]$ ($d = 2$, $\lambda = 1/4$)}
  \label{fig:sigma}
\end{figure}

\begin{table}
  \caption{Some computed values of $\sigma$ ($d = 2$, $\lambda = 1/4$)}
  \label{tbl:sigma}
  \begin{tabular}{|c|c|c|c|c|c|c|c|c|c|} \hline
    $t$ & $\sigma$ & $t$ & $\sigma$ & $t$ & $\sigma$ & $t$ & $\sigma$ & $t$ & $\sigma$ \\ \hline
    $0.0$ & $0.37462$ & $4.0$ & $0.95210$ & $8.0$ & $0.97394$ & $12.0$ & $0.97662$ & $16.0$ & $0.96739$ \\ \hline
    $0.4$ & $0.44248$ & $4.4$ & $0.95146$ & $8.4$ & $0.96359$ & $12.4$ & $0.97848$ & $16.4$ & $0.96309$ \\ \hline
    $0.8$ & $0.53832$ & $4.8$ & $0.95003$ & $8.8$ & $0.96359$ & $12.8$ & $0.97233$ & $16.8$ & $0.96309$ \\ \hline
    $1.2$ & $0.67932$ & $5.2$ & $0.95003$ & $9.2$ & $0.96314$ & $13.2$ & $0.97204$ & $17.2$ & $0.96307$ \\ \hline
    $1.6$ & $0.87394$ & $5.6$ & $0.94924$ & $9.6$ & $0.95312$ & $13.6$ & $0.97061$ & $17.6$ & $0.96067$ \\ \hline
    $2.0$ & $0.95210$ & $6.0$ & $0.94787$ & $10.0$ & $0.95325$ & $14.0$ & $0.96739$ & $18.0$ & $0.95879$ \\ \hline
    $2.4$ & $0.95210$ & $6.4$ & $0.94891$ & $10.4$ & $0.95792$ & $14.4$ & $0.96565$ & $18.4$ & $0.95962$ \\ \hline
    $2.8$ & $0.95210$ & $6.8$ & $0.95204$ & $10.8$ & $0.96260$ & $14.8$ & $0.96478$ & $18.8$ & $0.96209$ \\ \hline
    $3.2$ & $0.95210$ & $7.2$ & $0.95725$ & $11.2$ & $0.96727$ & $15.2$ & $0.96478$ & $19.2$ & $0.96621$ \\ \hline
    $3.6$ & $0.95210$ & $7.6$ & $0.96455$ & $11.6$ & $0.97195$ & $15.6$ & $0.96565$ & $19.6$ & $0.97198$ \\ \hline
  \end{tabular}
\end{table}

\begin{figure}
  \includegraphics[width=0.75\textwidth]{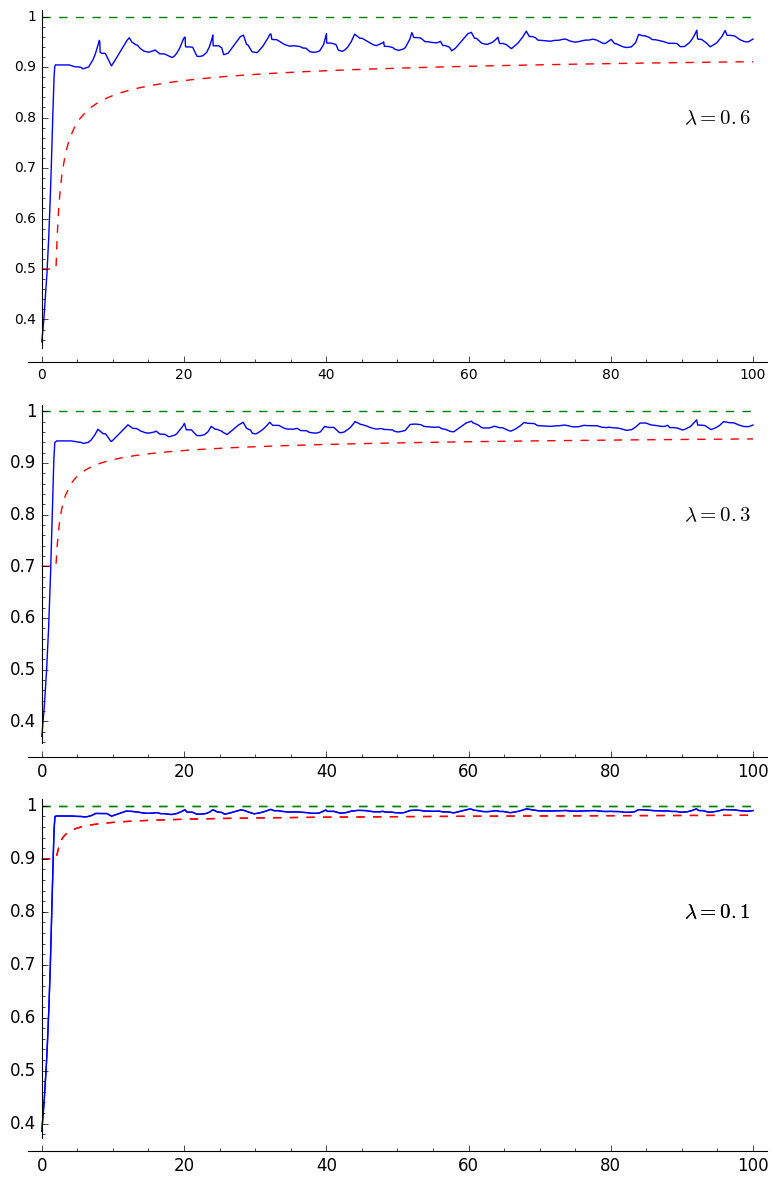}
  \caption{The graph of $\sigma$ on $[0, 100]$ ($d = 2$)}
  \label{fig:sigma100}
\end{figure}

\section{Main results} \label{sec:main}

The main results of the paper are formulated in the following two theorems.

\begin{theorem} \label{thm:1}
Assume that $f$ is a continuous function on a finite segment $[a, b]$ of $\mathbb{R}$ and $\sigma$ is the sigmoidal function constructed in Section~\ref{sec:construction}. Then for any sufficiently small $\varepsilon > 0$ there exist constants $c_1$, $c_2$, $\theta_1$ and $\theta_2$ such that
\begin{equation*}
  | f(x) - c_1 \sigma(x - \theta_1) - c_2 \sigma(x - \theta_2) | < \varepsilon
\end{equation*}
for all $x \in [a, b]$.
\end{theorem}
\begin{proof}
Set $d := b - a$ and divide the interval $[d, +\infty)$ into the segments $[d, 2d]$, $[2d, 3d]$, $\ldots$. It follows from~(\ref{eq:sigma_again}) that
\begin{equation} \label{eq:sigma_d}
  \sigma(d x + (2n-1) d) = a_n + b_n u_n(x), \qquad x \in [0, 1]
\end{equation}
for $n = 1$, $2$, $\ldots$. Here $a_n$ and $b_n$ are computed by~(\ref{eq:a_1, b_1}) and~(\ref{eq:a_n, b_n}) for $n=1$ and $n>1$, respectively.

From~(\ref{eq:sigma_d}) it follows that for each $n = 1$, $2$, $\ldots$,
\begin{equation} \label{eq:u_d}
  u_n(x) = \frac{1}{b_n} \sigma(d x + (2n-1) d) - \frac{a_n}{b_n}.
\end{equation}

Let now $g$ be any continuous function on the unit interval $[0, 1]$. By the density of polynomials with rational coefficients in the space of continuous functions over any compact subset of $\mathbb{R}$, for any $\varepsilon > 0$ there exists a polynomial $p(x)$ of the above form such that
\begin{equation*}
  |g(x) - p(x)| < \varepsilon
\end{equation*}
for all $x \in [0, 1]$. Denote by $p_0$ the leading coefficient of $p$. If $p_0 \ne 0$ (i.e., $p \not\equiv 0$) then we define $u_n$ as $u_n(x) := p(x) / p_0$, otherwise we just set $u_n(x) := 1$. In both cases
\begin{equation*}
  |g(x) - p_0 u_n(x)| < \varepsilon, \qquad x \in [0, 1].
\end{equation*}
This together with~(\ref{eq:u_d}) means that
\begin{equation*}
  |g(x) - c_1 \sigma(d x - s_1) - c_0| < \varepsilon
\end{equation*}
for some $c_0$, $c_1$, $s_1 \in \mathbb{R}$ and all $x \in [0, 1]$. Namely, $c_1 = p_0 / b_n$, $s_1 = d - 2nd$ and $c_0 = p_0 a_n / b_n$. On the other hand, we can write $c_0 = c_2 \sigma(d x - s_2)$, where $c_2 := 2 c_0 / (1 + h(3d))$ and $s_2 := -d$. Hence,
\begin{equation} \label{eq:g_c_1_c_2}
  |g(x) - c_1 \sigma(d x - s_1) - c_2 \sigma(d x - s_2)| < \varepsilon.
\end{equation}
Note that~(\ref{eq:g_c_1_c_2}) is valid for the unit interval $[0, 1]$. Using linear transformation it is not difficult to go from $[0, 1]$ to the interval $[a, b]$. Indeed, let $f \in C[a, b]$, $\sigma$ be constructed as above, and $\varepsilon$ be an arbitrarily small positive number. The transformed function $g(x) = f(a + (b - a) x)$ is well defined on $[0, 1]$ and we can apply the inequality~(\ref{eq:g_c_1_c_2}). Now using the inverse transformation $x = (t - a) / (b - a)$, we can write
\begin{equation*}
  |f(t) - c_1 \sigma(t - \theta_1) - c_2 \sigma(t - \theta_2)| < \varepsilon
\end{equation*}
for all $t \in [a, b]$, where $\theta_1 = a + s_1$ and $\theta_2 = a + s_2$. The last inequality completes the proof.
\end{proof}

Since any compact subset of the real line is contained in a segment $[a, b]$, the following generalization of Theorem~\ref{thm:1} holds.

\begin{theorem} \label{thm:2}
Let $Q$ be a compact subset of the real line and $d$ be its diameter. Let $\lambda$ be any positive number. Then one can algorithmically construct a computable sigmoidal activation function $\sigma \colon \mathbb{R} \to \mathbb{R}$, which is infinitely differentiable, strictly increasing on $(-\infty, d)$, $\lambda$-strictly increasing on $[d,+\infty)$, and satisfies the following property: For any $f \in C(Q)$ and $\varepsilon > 0$ there exist numbers $c_1$, $c_2$, $\theta_1$ and $\theta_2$ such that
\begin{equation*}
  |f(x) - c_1 \sigma(x - \theta_1) - c_2 \sigma(x - \theta_2)| < \varepsilon
\end{equation*}
for all $x \in Q$.
\end{theorem}

\begin{remark}
The idea of using monic polynomials (see Section~\ref{sec:construction} and the proof above) is new in the numerical analysis of neural networks with limited number of hidden neurons. In fact, if one is interested more in a theoretical than in a practical result, then any countable dense subset of $C[0, 1]$ suffices. Maiorov and Pinkus \cite{MP99} used such a subset to prove existence of a sigmoidal, monotonic and analytic activation function, and consequently a neural network with a fixed number of hidden neurons, which approximates arbitrarily well any continuous function. Note that the result is of theoretical value and the authors of \cite{MP99} do not suggest constructing and using their sigmoidal function. In our previous work \cite{GI16}, we exploited a sequence of all polynomials with rational coefficients to construct a new universal sigmoidal function. Note that in \cite{GI16} the problem of fixing weights in approximation by neural networks was not considered. Although the construction was efficient in the sense of computation of that sigmoidal function, some difficulties appeared while computing an approximating neural network parameters for some relatively simple approximated functions (see Remark~2 in~\cite{GI16}). This was a reason why we avoided giving practical numerical examples. The usage of monic polynomials in this instance turned out to be advantageous in reducing ``running time'' of the algorithm for computing the mentioned network parameters. This allows one to approximate various functions with sufficiently small precision and obtain all the required parameters (scaling coefficients and thresholds) in practice. We give corresponding numerical results in the next section.
\end{remark}

\section{Numerical results} \label{sec:numerical}

We prove in Theorem~\ref{thm:1} that any continuous function on $[a, b]$ can be approximated arbitrarily well by SLFNs with the fixed weight $1$ and with only two neurons in the hidden layer. An activation function $\sigma$ for such a network is constructed in Section~\ref{sec:construction}. We have seen from the proof that our approach is totally constructive. One can evaluate the value of $\sigma$ at any point of the real axis and draw its graph instantly using the programming interface at the URL shown at the beginning of Section~\ref{sec:computation}. In the current section, we demonstrate our result in various examples. For different error bounds we find the parameters $c_1$, $c_2$, $\theta_1$ and $\theta_2$ in Theorem~\ref{thm:1}. All computations were done in SageMath~\cite{Sage}. For computations, we use the following algorithm, which works well for analytic functions. Assume $f$ is a function, whose Taylor series around the point $(a + b) / 2$ converges uniformly to $f$ on $[a, b]$, and $\varepsilon > 0$.
\begin{enumerate}
  \item Consider the function $g(t) := f(a + (b - a) t)$, which is well-defined on $[0, 1]$;
  \item Find $k$ such that the $k$-th Taylor polynomial
\begin{equation*}
  T_k(x) := \sum_{i=0}^k \frac{g^{(i)}(1/2)}{i!} \left( x - \frac{1}{2} \right)^i
\end{equation*}
satisfies the inequality $|T_k(x) - g(x)| \le \varepsilon / 2$ for all $x \in [0, 1]$;
  \item Find a polynomial $p$ with rational coefficients such that
\begin{equation*}
  |p(x) - T_k(x)| \le \frac{\varepsilon}{2}, \qquad x \in [0, 1],
\end{equation*}
and denote by $p_0$ the leading coefficient of this polynomial;
  \item If $p_0 \ne 0$, then find $n$ such that $u_n(x) = p(x) / p_0$. Otherwise, set $n := 1$;
  \item For $n=1$ and $n > 1$ evaluate $a_n$ and $b_n$ by~(\ref{eq:a_1, b_1}) and~(\ref{eq:a_n, b_n}), respectively;
  \item Calculate the parameters of the network as
\begin{equation*}
  c_1 := \frac{p_0}{b_n}, \qquad c_2 := \frac{2 p_0 a_n}{b_n (1 + h(3d))}, \qquad \theta_1 := b - 2 n (b - a), \qquad \theta_2 := 2 a - b;
\end{equation*}

  \item Construct the network $\mathcal{N}=c_{1}\sigma (x-\theta _{1})+c_{2}\sigma (x-\theta _{2}).$ Then $\mathcal{N}$ gives an $\varepsilon
$-approximation to $f.$
\end{enumerate}

In the sequel, we give four practical examples. To be able to make comparisons between these examples, all the considered functions are given on the same interval $[-1, 1]$. First we select the polynomial function $f(x)=x^{3}+x^{2}-5x+3$ as a target function. We investigate the sigmoidal neural network approximation to $f(x)$. This function was also considered in \cite{HH04}. Note that in \cite{HH04} the authors chose the sigmoidal function as
\begin{equation*}
  \sigma(x) = \begin{cases}
    1, & \text{if } x \ge 0, \\
    0, & \text{if } x < 0,
  \end{cases}
\end{equation*}
and obtained the numerical results (see Table~\ref{tbl:Hahm}) for SLFNs with $8$, $32$, $128$, $532$ neurons in the hidden layer (see also~\cite{CX10} for an additional constructive result concerning the error of approximation in this example).

\begin{table}
  \caption{The Heaviside function as a sigmoidal function}
  \label{tbl:Hahm}
  \begin{tabular}{|c|c|c|} \hline
    $N$ & Number of neurons ($2N^2$) & Maximum error \\ \hline
    $2$ & $8$ & $0.666016$ \\ \hline
    $4$ & $32$ & $0.165262$ \\ \hline
    $8$ & $128$ & $0.041331$ \\ \hline
    $16$ & $512$ & $0.010333$ \\ \hline
  \end{tabular}
\end{table}

As it is seen from the table, the number of neurons in the hidden layer increases as the error bound decreases in value. This phenomenon is no longer true for our sigmoidal function (see Section~\ref{sec:construction}). Using Theorem~\ref{thm:1}, we can construct explicitly an SLFN with only two neurons in the hidden layer, which approximates the above polynomial with arbitrarily given precision. Here by \emph{explicit construction} we mean that all the network parameters can be computed directly. Namely, the calculated values of these parameters are as follows: $c_1 \approx 2059.373597$, $c_2 \approx -2120.974727$, $\theta_1 = -467$, and $\theta_2 = -3$. It turns out that for the above polynomial we have an exact representation. That is, on the interval $[-1, 1]$ we have the identity
$$
  x^3 + x^2 - 5x + 3 \equiv c_1 \sigma(x - \theta_1) + c_2 \sigma(x - \theta_2).
$$

Let us now consider the other polynomial function
$$
  f(x) = 1 + x + \frac{x^2}{2} + \frac{x^3}{6} + \frac{x^4}{24} + \frac{x^5}{120} + \frac{x^6}{720}.
$$
For this function we do not have an exact representation as above. Nevertheless, one can easily construct a $\varepsilon$-approximating network with two neurons in the hidden layer for any sufficiently small approximation error $\varepsilon$. Table~\ref{tbl:polynomial} displays numerical computations of the network parameters for six different approximation errors.

\begin{table}
  \caption{Several $\varepsilon$-approximators of the function $1 + x + x^2 / 2 + x^3 / 6 + x^4 / 24 + x^5 / 120 + x^6 / 720$}
  \label{tbl:polynomial}
  \begin{tabular}{|c|c|c|l|c|c|} \hline
    Number of & \multicolumn{4}{|c|}{Parameters of the network} & Maximum \\ \cline{2-5}
    neurons & $c_1$ & $c_2$ & \multicolumn{1}{c|}{$\theta_1$} & $\theta_2$ & error \\ \hline
    $2$ & $2.0619 \times 10^{2}$ & $2.1131 \times 10^{2}$ & $-1979$ & $-3$ & $0.95$ \\ \hline
    $2$ & $5.9326 \times 10^{2}$ & $6.1734 \times 10^{2}$ & $-1.4260 \times 10^{8}$ & $-3$ & $0.60$ \\ \hline
    $2$ & $1.4853 \times 10^{3}$ & $1.5546 \times 10^{3}$ & $-4.0140 \times 10^{22}$ & $-3$ & $0.35$ \\ \hline
    $2$ & $5.1231 \times 10^{2}$ & $5.3283 \times 10^{2}$ & $-3.2505 \times 10^{7}$ & $-3$ & $0.10$ \\ \hline
    $2$ & $4.2386 \times 10^{3}$ & $4.4466 \times 10^{3}$ & $-2.0403 \times 10^{65}$ & $-3$ & $0.04$ \\ \hline
    $2$ & $2.8744 \times 10^{4}$ & $3.0184 \times 10^{4}$ & $-1.7353 \times 10^{442}$ & $-3$ & $0.01$ \\ \hline
  \end{tabular}
\end{table}

\begin{figure}
  \includegraphics[width=1.0\textwidth]{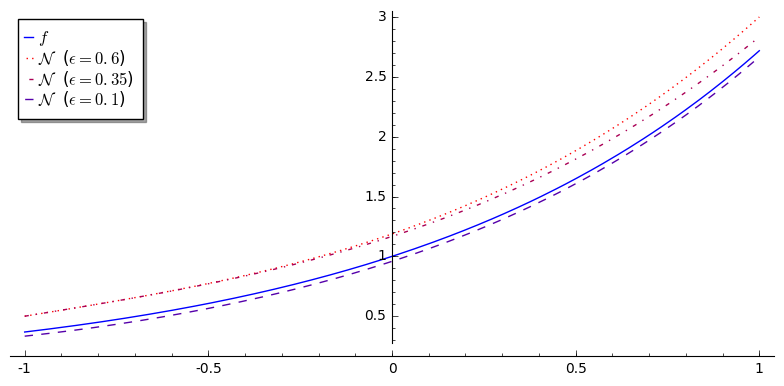}
  \caption{The graphs of $f(x) = 1 + x + x^2 / 2 + x^3 / 6 + x^4 / 24 + x^5 / 120 + x^6 / 720$ and some of its approximators ($\lambda = 1/4$)}
  \label{fig:polynomial}
\end{figure}

At the end we consider the nonpolynomial functions $f(x) = 4x / (4 + x^2)$ and $f(x) = \sin x - x \cos(x + 1)$. Tables~\ref{tbl:nonpolynomial}  and \ref{tbl:nonpolynomial2} display all the parameters of the $\varepsilon$-approximating neural networks for the above six approximation error bounds. As it is seen from the tables, these bounds do not alter the number of hidden neurons. Figures~\ref{fig:polynomial}, \ref{fig:nonpolynomial} and \ref{fig:nonpolynomial2} show how graphs of some constructed networks $\mathcal{N}$ approximate the corresponding target functions $f$.

\begin{table}
  \caption{Several $\varepsilon$-approximators of the function $4x / (4 + x^2)$}
  \label{tbl:nonpolynomial}
  \begin{tabular}{|@{\hspace{4pt}}c@{\hspace{4pt}}|c|c|l|c|c|} \hline
    Number of & \multicolumn{4}{|c|}{Parameters of the network} & Maximum \\ \cline{2-5}
    neurons & $c_1$ & $c_2$ & \multicolumn{1}{c|}{$\theta_1$} & $\theta_2$ & error \\ \hline
    $2$ & $\phantom{-}1.5965 \times 10^{2}$ & $\phantom{-}1.6454 \times 10^{2}$ & $-283$ & $-3$ & $0.95$ \\ \hline
    $2$ & $\phantom{-}1.5965 \times 10^{2}$ & $\phantom{-}1.6454 \times 10^{2}$ & $-283$ & $-3$ & $0.60$ \\ \hline
    $2$ & $-1.8579 \times 10^{3}$ & $-1.9428 \times 10^{3}$ & $-6.1840 \times 10^{11}$ & $-3$ & $0.35$ \\ \hline
    $2$ & $\phantom{-}1.1293 \times 10^{4}$ & $\phantom{-}1.1842 \times 10^{4}$ & $-4.6730 \times 10^{34}$ & $-3$ & $0.10$ \\ \hline
    $2$ & $\phantom{-}2.6746 \times 10^{4}$ & $\phantom{-}2.8074 \times 10^{4}$ & $-6.8296 \times 10^{82}$ & $-3$ & $0.04$ \\ \hline
    $2$ & $-3.4218 \times 10^{6}$ & $-3.5939 \times 10^{6}$ & $-2.9305 \times 10^{4885}$ & $-3$ & $0.01$ \\ \hline
  \end{tabular}
\end{table}

\begin{table}
  \caption{Several $\varepsilon$-approximators of the function $\sin x - x \cos(x + 1)$}
  \label{tbl:nonpolynomial2}
  \begin{tabular}{|c|c|c|l|c|c|} \hline
    Number of & \multicolumn{4}{|c|}{Parameters of the network} & Maximum \\ \cline{2-5}
    neurons & $c_1$ & $c_2$ & \multicolumn{1}{c|}{$\theta_1$} & $\theta_2$ & error \\ \hline
    $2$ & $\phantom{-}8.950 \times 10^{3}$ & $\phantom{-}9.390 \times 10^{3}$ & $-3.591 \times 10^{53}$ & $-3$ & $0.95$ \\ \hline
    $2$ & $\phantom{-}3.145 \times 10^{3}$ & $\phantom{-}3.295 \times 10^{3}$ & $-3.397 \times 10^{23}$ & $-3$ & $0.60$ \\ \hline
    $2$ & $\phantom{-}1.649 \times 10^{5}$ & $\phantom{-}1.732 \times 10^{5}$ & $-9.532 \times 10^{1264}$ & $-3$ & $0.35$ \\ \hline
    $2$ & $-4.756 \times 10^{7}$ & $-4.995 \times 10^{7}$ & $-1.308 \times 10^{180281}$ & $-3$ & $0.10$ \\ \hline
    $2$ & $-1.241 \times 10^{7}$ & $-1.303 \times 10^{7}$ & $-5.813 \times 10^{61963}$ & $-3$ & $0.04$ \\ \hline
    $2$ & $\phantom{-}1.083 \times 10^{9}$ & $\phantom{-}1.138 \times 10^{9}$ & $-2.620 \times 10^{5556115}$ & $-3$ & $0.01$ \\ \hline
  \end{tabular}
\end{table}

\begin{figure}
  \includegraphics[width=1.0\textwidth]{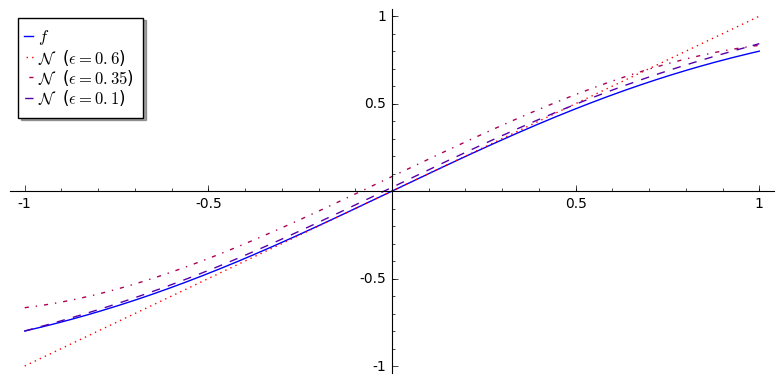}
  \caption{The graphs of $f(x) = 4x / (4 + x^2)$ and some of its approximators ($\lambda = 1/4$)}
  \label{fig:nonpolynomial}
\end{figure}

\begin{figure}
  \includegraphics[width=1.0\textwidth]{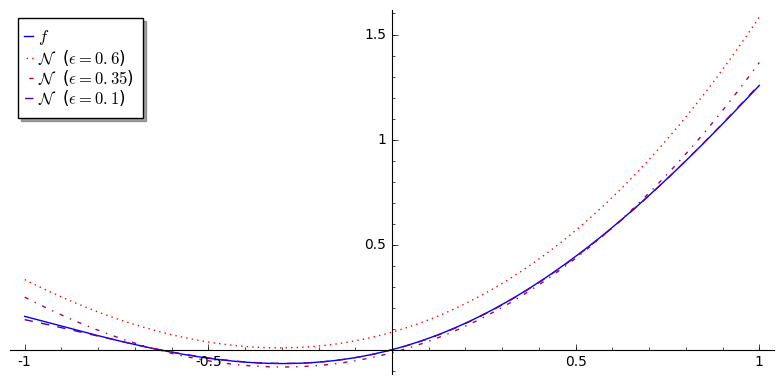}
  \caption{The graphs of $f(x) = \sin x - x \cos(x + 1)$ and some of its approximators ($\lambda = 1/4$)}
  \label{fig:nonpolynomial2}
\end{figure}

\section{Analysis of the multivariate case} \label{sec:multivariate}

In this section, we want to draw the reader's attention to the following question. Do SLFNs with fixed weights preserve their universal approximation property in the multivariate setting? That is, if networks of the form
\begin{equation} \label{eq:an}
  h(\mathbf{x}) = \sum_{i=1}^r c_i \sigma(\mathbf{w} \cdot \mathbf{x} - \theta_i),
\end{equation}
where the weight $\mathbf{w} \in \mathbb{R}^d$ is fixed for all units of the hidden layer, but which may be different for different networks $h$, can approximate any continuous multivariate function $f(x_1, \ldots, x_d)$, $d > 1$, within arbitrarily small tolerance? Note that if $\mathbf{w}$ is fixed for all $h$, then it is obvious that there is a multivariate function which cannot be approximated by networks of the form (\ref{eq:an}). Indeed, the linear functional
\begin{equation*}
  F(f) = f(x_1) - f(x_2),
\end{equation*}
where $\mathbf{x}_{1}$ and $\mathbf{x}_{2}$ are selected so that $\mathbf{w} \cdot \mathbf{x}_{1} = \mathbf{w} \cdot \mathbf{x}_{2}$, annihilates all functions $h$. Since the functional $F$ is nontrivial, the set of all functions $h$, which we denote in the sequel by $\mathcal{H}$, is not dense in $C(Q)$ for an arbitrary compact set $Q$ containing the points $\mathbf{x}_{1}$ and $\mathbf{x}_{2}$; hence approximation to all continuous functions cannot be possible on such compact sets $Q$. The above question, in the case where $\mathbf{w}$ is different for different networks $h$, is rather complicated. The positive answer to this question would mean, for example, that Theorem~\ref{thm:HH} admits a generalization to $d$-variable functions. Unfortunately, our answer to this question is negative. The details are as follows. Each summand in~(\ref{eq:an}) is a function depending on the inner product $\mathbf{w} \cdot \mathbf{x}$. Thus, the whole sum itself, i.e. the function $h(\mathbf{x})$ is a function of the form $g(\mathbf{w} \cdot \mathbf{x})$. Note that functions of the form $g(\mathbf{w} \cdot \mathbf{x})$ are called \emph{ridge functions}. The literature abounds with the use of such functions and their linear combinations (see, e.g., \cite{I16, P15} and a great deal of references therein). We see that the set $\mathcal{H}$ is a subset of the set of ridge functions $\mathcal{R} := \left\{ g(\mathbf{w} \cdot \mathbf{x}) \colon \mathbf{w} \in \mathbb{R}^{d} \setminus \{ \mathbf{0} \},\ g \in C(\mathbb{R}) \right\}$. Along with $\mathcal{R}$, let us also consider the sets
\begin{equation*}
  \mathcal{R}_{k} := \left\{ \sum_{i=1}^{k} g_{i}(\mathbf{w}^{i} \cdot \mathbf{x}) \colon
\mathbf{w}^{i} \in \mathbb{R}^{d} \setminus \{ \mathbf{0} \},\ g_{i} \in C(\mathbb{R}),\ i = 1, \ldots, k \right\}.
\end{equation*}
Note that in $\mathcal{R}_{k}$ we vary over both the vectors $\mathbf{w}^{i}$ and the functions $g_{i}$, whilst $k$ is fixed. Clearly, $\mathcal{R} = \mathcal{R}_{1}$. In \cite{LP93}, Lin and Pinkus proved that for any $k \in \mathbb{N}$, there exists a function $f \in C(\mathbb{R}^{d})$ and a compact set $Q \subset \mathbb{R}^{d}$ such that
\begin{equation*}
  \inf_{g \in \mathcal{R}_{k}} \left\Vert f - g \right\Vert > 0.
\end{equation*}
Here $\left\Vert \cdot \right\Vert $ denotes the uniform norm. It follows from this result that for each $k \in \mathbb{N}$ the set $\mathcal{R}_{k}$ (hence $\mathcal{R}$) is not dense in $C(\mathbb{R}^{d})$ in the topology of uniform convergence on compacta. Since $\mathcal{H} \subset \mathcal{R}$, we obtain that the set $\mathcal{H}$ cannot be dense either. Thus there are always continuous multivariate functions which cannot be approximated arbitrarily well by SLFNs with fixed weights. This phenomenon justifies why we and the other researchers (see Introduction) investigate universal approximation property of such networks only in the univariate case.

The above analysis leads us to the following general negative result on the approximation by SLFNs with limited weights.
\begin{theorem}
For any continuous function $\sigma \colon \mathbb{R\rightarrow R}$, there is a multivariate continuous function which cannot be approximated arbitrarily well by neural networks of the form
\begin{equation} \label{eq:an2}
  \sum_{i=1}^{r} c_{i} \sigma(\mathbf{w}^{i} \cdot \mathbf{x} - \theta_{i}),
\end{equation}
where we vary over all $r\in \mathbb{N}$, $c_{i},\theta _{i}\in \mathbb{R}$, $\mathbf{w}^{i}\in \mathbb{R}^{d}$, but the  number of pairwise independent vectors (weights) $\mathbf{w}^{i}$ in each network (\ref{eq:an2}) is uniformly bounded by some positive integer $k$ (which is the same for all networks).
\end{theorem}
This theorem shows a particular limitation of neural networks with one hidden layer. We refer the reader to \cite{CLM96, L17} for interesting results and discussions around other limitations of such networks.

\section*{Acknowledgements}

The research of the second author was supported by the Azerbaijan National Academy of Sciences under the program ``Approximation by neural networks and some problems of frames''.


\begin{thebibliography}{99}

\bibitem{CW00} N. Calkin and H. S. Wilf,
\emph{Recounting the rationals},
Amer. Math. Monthly \textbf{107} (2000), 360--367.

\bibitem{CX10} F. Cao and T. Xie,
\emph{The construction and approximation for feedforword neural networks with fixed weights},
Proceedings of the ninth international conference on machine learning and cybernetics, Qingdao, 2010, pp. 3164--3168.

\bibitem{CC93} T. Chen and H. Chen,
\emph{Approximation of continuous functionals by neural networks with application to dynamic systems},
IEEE Trans. Neural Networks \textbf{4} (1993), 910--918.

\bibitem{CL92} C. K. Chui and X. Li,
\emph{Approximation by ridge functions and neural networks with one hidden layer},
J. Approx. Theory \textbf{70} (1992), 131--141.

\bibitem{CLM96} C. K. Chui, X. Li and H. N. Mhaskar,
\emph{Limitations of the approximation capabilities of neural networks with one hidden layer},
Adv. Comput. Math. \textbf{5} (1996), no. 2-3, 233--243.

\bibitem{CS13} D. Costarelli and R. Spigler,
\emph{Constructive approximation by superposition of sigmoidal functions},
Anal. Theory Appl. \textbf{29} (2013), 169--196.

\bibitem{C90} N. E. Cotter,
\emph{The Stone--Weierstrass theorem and its application to neural networks},
IEEE Trans. Neural Networks \textbf{1} (1990), 290--295.

\bibitem{C89} G. Cybenko,
\emph{Approximation by superpositions of a sigmoidal function},
Math. Control Signal Systems \textbf{2} (1989), 303--314.

\bibitem{D02} S. Draghici,
\emph{On the capabilities of neural networks using limited precision weights},
Neural Networks \textbf{15} (2002), 395--414.

\bibitem{F89} K. Funahashi,
\emph{On the approximate realization of continuous mapping by neural networks},
Neural Networks \textbf{2} (1989), 183--192.

\bibitem{GW88} A. R. Gallant and H. White,
\emph{There exists a neural network that does not make avoidable mistakes},
Proceedings of the IEEE 1988 international conference on neural networks, vol. 1, IEEE Press, New York, 1988, pp. 657--664.

\bibitem{GI16} N. J. Guliyev and V. E. Ismailov,
\emph{A single hidden layer feedforward network with only one neuron in the hidden layer can approximate any univariate function},
Neural Computation \textbf{28} (2016), no. 7, 1289--1304.
\href{https://arxiv.org/abs/1601.00013}{arXiv:1601.00013}

\bibitem{HH04} N. Hahm and B.I. Hong,
\emph{An approximation by neural networks with a fixed weight},
Comput. Math. Appl. \textbf{47} (2004), no. 12, 1897--1903.

\bibitem{H91} K. Hornik,
\emph{Approximation capabilities of multilayer feedforward networks},
Neural Networks \textbf{4} (1991), 251--257.

\bibitem{IKM17} A. Iliev, N. Kyurkchiev and S. Markov,
\emph{On the approximation of the step function by some sigmoid functions},
Math. Comput. Simulation \textbf{133} (2017), 223--234.

\bibitem{I12} V. E. Ismailov,
\emph{Approximation by neural networks with weights varying on a finite set of directions},
J. Math. Anal. Appl. \textbf{389} (2012), no. 1, 72--83.

\bibitem{I15} \bysame,
\emph{Approximation by ridge functions and neural networks with a bounded number of neurons},
Appl. Anal. \textbf{94} (2015), no. 11, 2245--2260.

\bibitem{I16} \bysame,
\emph{Approximation by sums of ridge functions with fixed directions} (Russian),
Algebra i Analiz, \textbf{28} (2016), no. 6, 20--69.

\bibitem{IS17} V. E. Ismailov and E. Savas,
\emph{Measure theoretic results for approximation by neural networks with limited weights},
Numer. Funct. Anal. Optim. \textbf{38} (2017), no. 7, 819--830.

\bibitem{I92} Y. Ito,
\emph{Approximation of continuous functions on $\mathbb{R}^d$ by linear combinations of shifted rotations of a sigmoid function with and without scaling},
Neural Networks \textbf{5} (1992), 105--115.

\bibitem{JYJ10} B. Jian, C. Yu and Y. Jinshou,
\emph{Neural networks with limited precision weights and its application in embedded systems},
Proceedings of the the second international workshop on education technology and computer science, Wuhan, 2010, pp. 86--91.

\bibitem{K92} V. K\r{u}rkov\'{a},
\emph{Kolmogorov's theorem and multilayer neural networks},
Neural Networks \textbf{5} (1992), 501--506.

\bibitem{LLPS93} M. Leshno, V. Ya. Lin, A. Pinkus and S. Schocken,
\emph{Multilayer feedforward networks with a non-polynomial activation function can approximate any function},
Neural Networks \textbf{6} (1993), 861--867.

\bibitem{LFN04} Y. Liao, S.-C. Fang and H. L. W. Nuttle,
\emph{A neural network model with bounded-weights for pattern classification},
Comput. Oper. Res. \textbf{31} (2004), 1411--1426.

\bibitem{L17} S. Lin,
\emph{Limitations of shallow nets approximation},
Neural Networks \textbf{94} (2017), 96--102.

\bibitem{LGCX13} S. Lin, X. Guo, F. Cao and Z. Xu,
\emph{Approximation by neural networks with scattered data}
Appl. Math. Comput. \textbf{224} (2013), 29--35.

\bibitem{LP93} V. Ya. Lin and A. Pinkus,
\emph{Fundamentality of ridge functions},
J. Approx. Theory \textbf{75} (1993), 295--311.

\bibitem{MP99} V. Maiorov and A. Pinkus,
\emph{Lower bounds for approximation by MLP neural networks},
Neurocomputing \textbf{25} (1999), 81--91.

\bibitem{MM92} H. N. Mhaskar and C. A. Micchelli,
\emph{Approximation by superposition of a sigmoidal function and radial basis functions},
Adv. Appl. Math. \textbf{13} (1992), 350--373.

\bibitem{P99} A. Pinkus,
\emph{Approximation theory of the MLP model in neural networks},
Acta numerica, 1999, Cambridge Univ. Press, Cambridge, 1999, pp. 143--195.

\bibitem{P15} \bysame,
\emph{Ridge functions},
Cambridge University Press, Cambridge, 2015.

\bibitem{S61} P. C. Sikkema,
\emph{Der Wert einiger Konstanten in der Theorie der Approximation mit Bernstein-Polynomen},
Numer. Math. \textbf{3} (1961), 107--116.

\bibitem{Sage} W. A. Stein et al.,
\emph{Sage Mathematics Software (Version 7.6)},
The Sage Developers, 2017, \url{http://www.sagemath.org}.

\bibitem{SW90} M. Stinchcombe and H. White,
\emph{Approximating and learning unknown mappings using multilayer feedforward networks with bounded weights},
Proceedings of the 1990 IEEE international joint conference on neural networks, vol. 3, IEEE, New York, 1990, pp. 7--16.

\end{thebibliography}
\end{document}